\documentclass{article}

\usepackage{amsmath,amsthm,amssymb}
\usepackage{graphicx}

\usepackage{natbib}

\usepackage{algorithm}
\usepackage{algorithmic}

\usepackage[accepted,arxiv]{aistats2014}

\def\R{\mathbb R}

\def\Pr{\mathrm{Pr}}

\DeclareMathOperator*{\argmax}{arg\,max}
\DeclareMathOperator*{\argmin}{arg\,min}

\def\A{\mathcal A}

\def\E{\mathcal E}

\def\Norm{\mathcal N}

\newtheorem{theorem}{Theorem}
\newtheorem{lemma}{Lemma}
\newtheorem{corollary}{Corollary}

\begin{document}

\runningauthor{}

\twocolumn[
\aistatstitle{Exploiting correlation and budget constraints in Bayesian
multi-armed bandit optimization}
\aistatsauthor{Matthew W.~Hoffman\And Bobak Shahriari\And Nando de Freitas}
\aistatsaddress{Cambridge University\And University of British Columbia\And Oxford University}
]

\begin{abstract}
We address the problem of finding the maximizer of a nonlinear smooth function,
that can only be evaluated point-wise, subject to constraints on the number of
permitted function evaluations.  This problem is also known as fixed-budget best
arm identification in the multi-armed bandit literature. We introduce a Bayesian
approach for this problem and show that it empirically outperforms both the
existing frequentist counterpart and other Bayesian optimization methods. The
Bayesian approach places emphasis on detailed modelling, including the modelling
of correlations among the arms. As a result, it can perform well in situations
where the number of arms is much larger than the number of allowed function
evaluation, whereas the frequentist counterpart is inapplicable. This feature
enables us to develop and deploy practical applications, such as automatic
machine learning toolboxes.  The paper presents comprehensive comparisons of the
proposed approach, Thompson sampling, classical Bayesian optimization
techniques, more recent Bayesian bandit approaches, and state-of-the-art best
arm identification methods. This is the first comparison of many of these
methods in the literature and allows us to examine the relative merits of their
different features.
\end{abstract}

\section{Introduction}

We address the problem of finding the maximizer of a nonlinear smooth function
$f: {\cal A} \mapsto \mathbb{R}$ which can only be evaluated point-wise. The
function need not be convex, its derivatives may not be known, and the function
evaluations will generally be corrupted by some form of noise. Importantly, we
are interested in functions that are typically expensive to evaluate. Moreover,
we will also assume a finite budget of $T$ function evaluations. This
fixed-budget global optimization problem can be treated within the framework of
sequential design. In this context, by allowing function queries $a_t\in\A$ to
depend on previous points and their corresponding function evaluations, the
algorithm must adaptively construct a sequence of queries (or actions) $a_{1:T}$
and afterwards return the element of highest expected value.

A typical example of this problem is that of automatic product testing
\citep{kohavi-abtesting, scott-bandits}, where common ``products'' correspond to
configuration options for ads, websites, mobile applications, and online games.
In this scenario, a company offers different product variations to a small
subset of customers, with the goal of finding the most successful product for
the entire customer base. The crucial problem is how best to query the smaller
subset of users in order to find the best product with high probability.  A
second example, analyzed later in this paper, is that of automating machine
learning. Here, the goal is to automatically select the best technique
(boosting, random forests, support vector machines, neural networks, etc.) and
its associated hyper-parameters for solving a machine learning task with a given
dataset. For big datasets, cross-validation is very expensive and hence it is
often important to find the best technique within a fixed budget of
cross-validation tests (function evaluations).

In order to properly attack this problem there are three design aspects that
must be considered. By taking advantage of \emph{correlation} among different
actions it is possible to learn more about a function than just its value at a
specific query. This is particularly important when the number of actions
greatly exceeds the \emph{finite query budget}. In this same vein, it is
important to take into account that a recommendation must be made at time $T$ in
order to properly allocate actions and explore the space of possible optima.
Finally, the fact that we are interested only in the value of the recommendation
made at time $T$ should be handled explicitly. In other words, we are only
interested in finding the \emph{best action} and are concerned with the rewards
obtained during learning only insofar as they inform us about this optimum.

In this work, we introduce a Bayesian approach that meets the above design goals
and show that it empirically outperforms the existing frequentist counterpart
\citep{gabillon-unified}. The Bayesian approach places emphasis on detailed
modelling, including the modelling of correlations among the arms. As a result,
it can perform well in situations where the number of arms is much larger than
the number of allowed function evaluation, whereas the frequentist counterpart
is inapplicable.  The paper presents comprehensive comparisons of the proposed
approach, Thompson sampling, classical Bayesian optimization techniques, more
recent Bayesian bandit approaches, and state-of-the-art best arm identification
methods. This is the first comparison of many of these methods in the literature
and allows us to examine the relative merits of their different features. The
paper also shows that one can easily obtain the same theoretical guarantees for
the Bayesian approach that were previously derived in the frequentist setting
\citep{gabillon-unified}.

\section{Related work}

Bayesian optimization has enjoyed success in a broad range of optimization
tasks; see the work of \cite{brochu-tutorial} for a broad overview.  Recently,
this approach has received a great deal of attention as a black-box technique
for the optimization of hyperparameters \citep{snoek:2012b, Hutter:smac,
Wang:rembo}.  This type of optimization combines prior knowledge about the
objective function with previous observations to estimate the posterior
distribution over $f$. The posterior distribution, in turn, is used to construct
an \emph{acquisition function} that determines what the next query point $a_t$
should be. Examples of acquisition functions include probability of improvement
(PI), expected improvement (EI), Bayesian upper confidence bounds (UCB), and
mixtures of these \citep{Mockus:1982, Jones:2001, Srinivas:2010, Hoffman:2011}.
One of the key strengths underlying the use of Bayesian optimization is the
ability to capture complicated correlation structures via the posterior
distribution.

Many approaches to bandits and Bayesian optimization focus on online learning
(\emph{e.g.}, minimizing cumulative regret) as opposed to optimization
\citep{Srinivas:2010,Hoffman:2011}. In the realm of optimizing deterministic
functions, a few works have proven exponential rates of convergence for simple
regret \citep{zoghi-detbo,Munos:2011}.  A stochastic variant of the work of
\citeauthor{Munos:2011} has been recently proposed by \cite{Valko:SSOO}; this
approach takes a tree-based structure for expanding areas of the optimization
problem in question, but it requires one to evaluate each cell many times before
expanding, and so may prove expensive in terms of the number of function
evaluations.

The problem of optimization under budget constraints has received relatively
little attention in the Bayesian optimization literature, though some approaches
without strong theoretical guarantees have been proposed recently
\citep{Azimi:2011, hennig-entropy,snoek-oppcost,villemonteix-iago}. In contrast,
optimization under budget constraints has been studied in significant depth in
the setting of multi-armed bandits \citep{bubeck-pure, audibert-best,
gabillon-multi, gabillon-unified}. Here, a decision maker must repeatedly choose
query points, often discrete and known as ``arms'', in order to observe their
associated rewards \citep{cesa-bianchi-book}. However, unlike most methods in
Bayesian optimization the underlying value of each action is generally assumed
to be independent from all other actions. That is, the correlation structure of
the arms is often ignored.

\section{Problem formulation}
\label{sec:problem}

In order to attack the problem of Bayesian optimization from a bandit
perspective we will consider a discrete collection of arms $\A=\{1,\dots,K\}$
such that the immediate reward of pulling arm $k\in\A$ is characterized by a
distribution $\nu_k$ with mean $\mu_k$. From the Bayesian optimization
perspective we can think of this as a collection of points $\{a_1,\dots,a_K\}$
where $\mu_k=f(a_k)$. Note that while we will assume the distributions $\nu_k$
are independent of past actions this \emph{does not} mean that the means of each
arm cannot share some underlying structure---only that the act of pulling arm
$k$ does not affect the future rewards of pulling this or any other arm. This
distinction will be relevant later in this section.

The problem of identifying the best arm in this bandit problem can now be
introduced as a sequential decision problem. At each round $t$ the decision
maker will select or ``pull'' an arm $a_t\in\A$ and observe an independent
sample $y_t$ drawn from the corresponding distribution $\nu_{a_t}$. At the
beginning of each round $t$, the decision maker must decide which arm to select
based only on previous interactions, which we will denote with the tuple
$(a_{1:t-1}, y_{1:t-1})$. For any arm $k$ we can also introduce the expected
immediate regret of selecting that arm as
\begin{equation}
    R_k = \mu^* - \mu_k,
\end{equation}
where $\mu^*$ denotes the expected value of the best arm. Note that while we
are interested in finding the arm with the minimum regret, the exact value of
this quantity is unknown to the learner.

In standard bandit problems the goal is generally to minimize the cumulative sum
of immediate regrets incurred by the arm selection process. Instead, in this
work we consider the \emph{pure exploration} setting \citep{bubeck-pure,
audibert-best}, which divides the sampling process into two phases: exploration
and evaluation. The exploration phase consists of $T$ rounds wherein a decision
maker interacts with the bandit process by sampling arms. After these rounds,
the decision maker must make a single arm recommendation $\Omega_T\in\A$. The
performance of the decision maker is then judged \emph{only} on the performance
of this recommendation. The expected performance of this single recommendation
is known as the \emph{simple regret}, and we can write this quantity as
$R_{\Omega_T}$. Given a tolerance $\epsilon>0$ we can also define the
\emph{probability of error} as the probability that $R_{\Omega_T}>\epsilon$. In
this work, we will consider both the empirical probability that our regret
exceeds some $\epsilon$ as well as the actual reward obtained.

\section{Bayesian bandits}
\label{sec:bayesian}

We will now consider a bandit problem wherein the distribution of rewards for
each arm is assumed to depend on unknown parameters $\theta\in\Theta$ that are
shared between all arms. We will write the reward distribution for arm $k$ as
$\nu_k(\cdot|\theta)$. When considering the bandit problem from a Bayesian
perspective, we will assume a prior density $\theta\sim \pi_0(\cdot)$ from which
the parameters are drawn. Next, after $t-1$ rounds we can write the posterior
density of these parameters as
\begin{equation}
    \pi_t(\theta)
    \propto \pi_0(\theta)
    \prod_{n<t} \nu_{a_n}(y_n|\theta).
\end{equation}
Here we can see the effect of choosing arm $a_n$ at each time $n$: we obtain
information about $\theta$ only indirectly by way of the likelihood of these
parameters given reward observations $y_n$.  Note that this also generalizes the
\emph{uncorrelated} arms setting.  If the rewards for each arm $k$ depend only
on a parameter (or set of parameters) $\theta_k$, then at time $t$ the posterior
for that parameter would only depend on those times in the past that we had
pulled arm $k$.

We are, however, only partially interested in the posterior distribution of the
parameters $\theta$. Instead, we are primarily concerned with the expected
reward for each arm under these parameters, which can be written as $\mu_k =
\mathbb E[Y|\theta] = \int y \,\nu_k(y|\theta) \,dy$. The true value of $\theta$
is unknown, but we have access to the posterior distribution $\pi_t(\theta)$.
This distribution induces a marginal distribution over $\mu_k$, which we will
write as $\rho_{kt}(\mu_k)$.  The distribution $\rho_{kt}(\mu_k)$ can then be
used to define upper and lower confidence bounds that hold with high probability
and, hence, engineer acquisition functions that trade-off exploration and
exploitation. We will derive an analytical expression for this distribution
next.

We will assume that each arm $k$ is associated with a feature vector
$x_k\in\R^d$ and where the rewards for pulling arm $k$ are normally distributed
according to
\begin{equation}
\label{eq:observation}
    \nu_k(y|\theta) = \Norm(y; x_k^T\theta, \sigma^2)
\end{equation}
with variance $\sigma^2$ and unknown $\theta\in\R^d$. The rewards for each arm
are independent conditioned on $\theta$, but marginally dependent when this
parameter is unknown. In particular the level of their dependence is given by
the structure of the vectors $x_k$. By placing a prior $\theta\sim
\Norm(0,\eta^2I)$ over the entire parameter vector we can compute a posterior
distribution over this unknown quantity. One can also easily place an
inverse-Gamma prior on $\sigma$ and compute the posterior analytically, but we
will not describe this in order to keep the presentation simple.

The above linear observation model might seem restrictive. However, because we
are only considering $K$ discrete actions (arms), it includes the Gaussian
process (GP) setting.  More precisely, let the matrix $G\in\R^{K\times K}$ be
the covariance of a GP prior. Our experiments will detail two ways of
constructing this covariance in practice. We can apply the following
transformation to construct the design matrix $X = [x_1\dots x_K]^T$:
\begin{equation*}
    X = VD^{\frac12},\ \text{where}\ G=VDV^T.
\end{equation*}
The rows of $X$ correspond to the vectors $x_k$ necessary for the construction
of the observation model in Equation~\eqref{eq:observation}.   By restricting
ourselves to discrete actions spaces, we can also implement strategies such a
Thompson sampling with GPs. The restriction to discrete action spaces poses some
scaling challenges in high-dimensions, but it enables us to deploy a broad set
of algorithms to attack low-dimensional problems. For this pragmatic reason,
many existing popular Bayesian optimization software tools consider discrete
actions only.

We will now let $X_t=[x_{a_1}\dots x_{a_{t-1}}]^T$ denote the design matrix and
$Y_t=[y_1 \dots y_{t-1}]^T$ the vector of observations at the beginning of round
$t$. We can then write the posterior at time $t$ as $\pi_t(\theta) =
\Norm(\theta; \hat\theta_t, \hat\Sigma_t)$, where
\begin{align}
    \hat\Sigma_t^{-1}
    &= \sigma^{-2}X_t^T X_t + \eta^{-2}I,
    \text{ and}\quad \\
    \hat\theta_t
    &= \sigma^{-2}\hat\Sigma_t X_t^TY_t.
\end{align}
From this formulation we can see that the expected reward associated with arm
$k$ is marginally normal $\rho_{kt}(\mu_k)=\mathcal N(\mu_k; \hat\mu_{kt}, \hat
\sigma^2_{kt})$ with mean $\hat\mu_{kt}=x_k^T\hat\theta_t$ and variance
$\hat\sigma_{kt}^2=x_k^T\hat\Sigma_t x_k$. Note also that the predictive
distribution over rewards associated with the $k$th arm is normal as well, with
mean $\hat\mu_{kt}$ and variance $\hat\sigma_{kt}^2+\sigma^2$. The previous
derivations are textbook material; see for example Chapter 7 of
\citep{Murphy:2012}.

\begin{figure}
\centering
\includegraphics[width=\columnwidth]{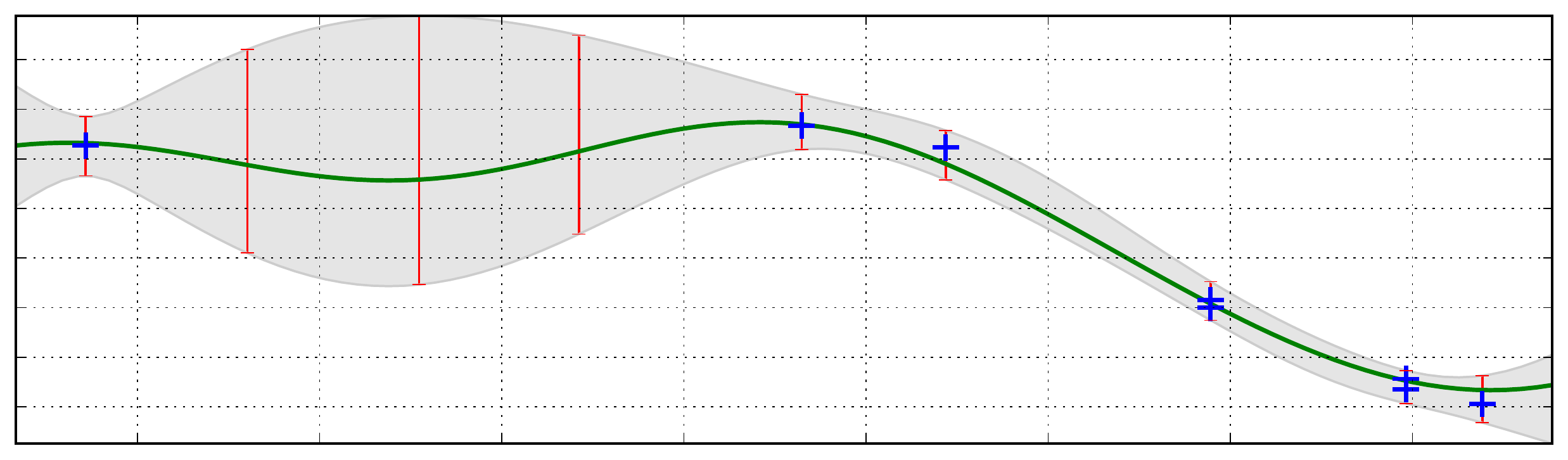}
\caption{\footnotesize Example GP setting with discrete arms. The full GP is
plotted with observations and confidence intervals at each of $K=10$ arms (mean
and confidence intervals of $\rho_{kt}(\mu_k)$). Shown in green is a single
sample from the GP.}
\label{fig:example}
\end{figure}

Figure~\ref{fig:example} depicts an example of the mean and confidence intervals
of $\rho_{kt}(\mu_k)$, as well as a single random sample. Here the features
$x_k$ were constructed by first forming the covariance matrix with an
exponential kernel $k(x,x')=\text e^{-(x-x')^2}$ over the 1-dimensional discrete
domain. As with standard Bayesian optimization with GPs, the statistics of
$\rho_{kt}(\mu_k)$ enable us to construct many different acquisition functions
that trade-off exploration and exploitation. Thompson sampling in this setting
also becomes straightforward, as we simply have to pick the maximum of the
random sample from $\rho_{kt}(\mu_k)$, at one of the discrete arms, as the next
point to query.

\section{Bayesian gap-based exploration}
\label{sec:gap}

In this section we will introduce a gap-based solution to the Bayesian
optimization problem, which we call BayesGap. This approach builds on the work
of \cite{gabillon-multi,gabillon-unified}, which we will refer to as
UGap\footnote{Technically this is UGapEb, denoting bounded horizon, but as we do
not consider the fixed-confidence variant in this paper we simplify the
acronym.}, and offers a principled way to incorporate correlation between
different arms (whereas the earlier approach assumes all arms are independent).

At the beginning of round $t$ we will assume that the decision maker is equipped
with high-probability upper and lower bounds $U_k(t)$ and $L_k(t)$ on the
unknown mean $\mu_k$ for each arm. While this approach can encompass more
general bounds, for the Gaussian-arms setting that we consider in this work we
can define these quantities in terms of the mean and standard deviation, i.e.\
$\hat\mu_{kt}\pm \beta\hat\sigma_{kt}$. These bounds also give rise to a
confidence diameter $s_k(t)=U_k(t)-L_k(t)=2\beta\hat\sigma_{kt}$.

Given bounds on the mean reward for each
arm, we can then introduce the gap quantity
\begin{equation}
    B_k(t) = \max_{i\neq k} U_i(t) - L_k(t),
\end{equation}
which involves a comparison between the lower bound of arm $k$ and the highest
upper bound among all alternative arms. Ultimately this quantity provides an
upper bound on the simple regret (see Lemma \ref{lemma:arm-regret-bound} in the
supplementary material) and will be used to define the exploration strategy.
However, rather than directly finding the arm minimizing this gap, we will
consider the two arms
\begin{align*}
    J(t) &= \argmin_{k\in\A} B_k(t) \text{ and} \\
    j(t) &= \argmax_{k\neq J(t)} U_k(t).
\end{align*}
We will then define the exploration strategy as
\begin{align}
    a_t &= \argmax_{k\in\{j(t), J(t)\}} s_k(t).
\end{align}
Intuitively this strategy will select either the arm minimizing our bound on the
simple regret (i.e.\ $J(t)$) or the best ``runner up'' arm. Between these two,
the arm with the highest uncertainty will be selected, i.e.\ the one expected to
give us the most information. Next, we will define the recommendation strategy
as
\begin{align}
    \Omega_T &= J\big( \argmin_{t\leq T} B_{J(t)}(t) \big),
    \label{eqn:budget-omega}
\end{align}
i.e.\ the proposal arm $J(t)$ which minimizes the regret bound, over all times
$t\leq T$. The reason behind this particular choice is subtle, but is necessary
for the proof of the method's simple regret bound\footnote{See inequality (b) in
the the supplementary material.}. In Algorithm~\ref{alg:gap} we show the
pseudo-code for BayesGap.

\begin{algorithm}[t]
\caption{BayesGap}
\label{alg:gap}
\begin{algorithmic}[1]
    \FOR{$t=1,\dots,T$}
        \STATE set $J(t) = \argmin_{k\in\A} B_k(t)$
        \STATE set $j(t) = \argmin_{k\neq J(t)} U_k(t)$
        \STATE select arm $a_t = \argmax_{k\in\{j(t),J(t)\}}s_k(t)$
        \STATE observe $y_t\sim\nu_{a_t}(\cdot)$
        \STATE update posterior $\hat\mu_{kt}$ and $\hat\sigma_{kt}$
        \STATE update bound on $H_\epsilon$ and re-compute $\beta$
        \STATE update posterior bounds $U_k(t)$ and $L_k(t)$
    \ENDFOR
    \RETURN $\Omega_T= J\big( \argmin_{t\leq T} B_{J(t)}(t) \big)$
\end{algorithmic}
\end{algorithm}

We now turn to the problem of which value of $\beta$ to use. First, consider the
quantity $\Delta_k=|\max_{i\neq k} \mu_i - \mu_k|$. For the best arm this
coincides with a measure of the distance to the second-best arm, whereas for all
other arms it is a measure of their sub-optimality. Given this quantity let
$H_{k\epsilon} = \max(\tfrac12 (\Delta_k+\epsilon), \epsilon)$ be an
arm-dependent hardness quantity; essentially our goal is to reduce the
uncertainty in each arm to below this level, at which point with high
probability we will identify the best arm. Now, given $H_\epsilon=\sum_k
H_{k\epsilon}^{-2}$ we define our exploration constant as
\begin{equation}
    \beta^2 =
    \big((T-K)/\sigma^2+\kappa/\eta^2\big) \big/ (4H_\epsilon)
    \label{eq:beta}
\end{equation}
where $\kappa=\sum_k \|x_k\|^{-2}$.  We have chosen $\beta$ such that with high
probability we recover an $\epsilon$-best arm, as detailed in the following
theorem. This theorem relies on bounding the uncertainty for each arm by a
function of the number of times that arm is pulled. Roughly speaking, if this
bounding function is monotonically decreasing and if the bounds $U_k$ and $L_k$
hold with high probability we can then apply Theorem~\ref{theorem:budget-bound}
to bound the simple regret of BayesGap\footnote{The additional Theorem is in
supplementary material and is a slight modification of that in
\citep{gabillon-unified}.}.

\begin{theorem}
    \label{cor:gaussian}
    Consider a $K$-armed Gaussian bandit problem, horizon $T$, and upper and
    lower bounds defined as above.  For $\epsilon>0$ and $\beta$ defined as in
    Equation (\ref{eq:beta}), the algorithm attains simple regret satisfying
    $\Pr(R_{\Omega_T}\leq\epsilon) \geq 1-KTe^{-\beta^2/2}$.
\end{theorem}
\begin{proof}
    Using the definition of the posterior variance for arm $k$, we can write the
    confidence diameter as
    \begin{align*}
        s_k(t)
        &= 2\beta\sqrt{x_k^T\hat\Sigma_t x_k} \\
        &=
        2\beta
        \sqrt{
            \sigma^2 x_k^T
            \big(\textstyle\sum_i N_i(t-1)\, x_ix_i^T + \tfrac{\sigma^2}{\eta^2}I
            \big)^{-1}
            x_k} \\
        &\leq
        2\beta
        \sqrt{
            \sigma^2 x_k^T
            \big(N_k(t-1)\, x_kx_k^T + \tfrac{\sigma^2}{\eta^2}I
            \big)^{-1}
            x_k}.
    \end{align*}
    In the second equality we decomposed the Gram matrix $X_t^TX_t$ in terms of
    a sum of outer products over the fixed vectors $x_i$. In the final
    inequality we noted that by removing samples we can only increase the
    variance term, i.e.\ here we have essentially replaced $N_i(t-1)$ with $0$
    for $i\neq k$. We will let the result of this final inequality define an
    arm-dependent bound $g_k$. Letting $A=\tfrac1N\tfrac{\sigma^2}{\eta^2}$ we
    can simplify this quantity using the Sherman-Morrison formula as
    \begin{align*}
        g_k(N)
        &=
        2\beta\sqrt{
            (\sigma^2/N)
            x_k^T \big(x_kx_k^T + AI\big)^{-1} x_k} \\
        &=
        2\beta
        \sqrt{
            \frac{\sigma^2}{N}
            \frac{\|x_k\|^2}{A}
            \Big(
            1 - \frac{\|x_k\|^2/A}{1+\|x_k\|^2/A}
            \Big)} \\
        &=
        2\beta
        \sqrt{\frac{\sigma^2\|x_k\|^2}{\tfrac{\sigma^2}{\eta^2} + N\|x_k\|^2}},
    \end{align*}
    which is monotonically decreasing in $N$. The inverse of this function can
    be solved for as
    \begin{equation*}
        g^{-1}_k(s) =
        \frac{4(\beta\sigma)^2}{s^2} -
        \frac{\sigma^2}{\eta^2}
        \frac{1}{\|x_k\|^2}.
    \end{equation*}
    By setting $\sum_k g_k^{-1}(H_{k\epsilon})=T-K$ and solving for $\beta$ we
    then obtain the definition of this term given in the statement of the
    proposition. Finally, by reference to Lemma~\ref{lemma:gaussian-deviation}
    (supplementary material) we can see that for each $k$ and $t$, the upper and
    lower bounds must hold with probability $1-e^{-\beta/2}$. These last two
    statements satisfy the assumptions of Theorem~\ref{theorem:budget-bound}
    (supplementary material), thus concluding our proof.
\end{proof}


Here we should note that while we are using Bayesian methodology to drive the
exploration of the bandit, we are analyzing this using frequentist regret
bounds. This is a common practice when analyzing the regret of Bayesian bandit
methods \citep{Srinivas:2010,kaufmann-bayesucb}. We should also point out that
implicitly Theorem~\ref{theorem:budget-bound} assumes that each arm is pulled at
least once regardless of its bound. However, in our setting we can avoid this in
practice due to the correlation between arms.

One key thing to note is that the proof and derivation of $\beta$ given above
explicitly require the hardness quantity $H_\epsilon$, which is unknown in most
practical applications. Instead of requiring this quantity, our approach will be
to adaptively estimate it. Intuitively, the quantity $\beta$ controls how much
exploration BayesGap does (note that $\beta$ directly controls the width of the
uncertainty $s_k(t)$). Further, $\beta$ is inversely proportional to
$H_\epsilon$. As a result, in order to initially encourage more exploration we
will lower bound the hardness quantity. In particular, we can do this by upper
bounding each $\Delta_k$ by using conservative, posterior dependent upper and
lower bounds on $\mu_k$. In this work we use three posterior standard deviations
away from the posterior mean, i.e.\ $\hat\mu_k(t)\pm3\hat\sigma_{kt}$. (We
emphasize that these are not the same as $L_k(t)$ and $U_k(t)$.) Then the upper
bound on $\Delta_k$ is simply
\begin{equation*}
    \hat\Delta_k = \max_{j\neq k} (\hat\mu_j + 3\hat\sigma_j)
    - (\hat\mu_k - 3\hat\sigma_k).
\end{equation*}
From this point we can recompute $H_\epsilon$ and in turn recompute $\beta$
(step 7 in the pseudocode). For all experiments we will use this adaptive
method.

\textbf{Comparison with UGap.} The method in this section provides a Bayesian
version of the UGap algorithm which modifies the bounds used in this earlier
algorithm's arm selection step. By modifying step 6 of the BayesGap pseudo-code
to use either Hoeffding or Bernstein bounds we can re-obtain the UGap algorithm.
Note, however, that in doing so UGap assumes independent arms with bounded
rewards.

We can now roughly compare UGap's probability of error, i.e.\
$O(KT\exp(-\frac{T-K}{H_\epsilon}))$, with that of BayesGap,
$O(KT\exp(-\frac{T-K+\kappa\sigma^2/\eta^2}{H_\epsilon\sigma^2}))$. We can see
that with minor differences, these bounds are of the same order. First, we can
ignore the additional $\sigma^2$ term as this quantity is primarily due to the
distinction between bounded and Gaussian-distributed rewards. The $\eta^2$ term
corresponds to the concentration of the prior, and we can see that the more
concentrated the prior is (smaller $\eta$) the faster this rate is.  Note,
however, that the proof of BayesGap's simple regret relies on the true rewards
for each arm being within the support of the prior, so one cannot increase the
algorithm's performance by arbitrarily adjusting the prior. Finally, the
$\kappa$ term is related to the linear relationship between different arms.
Additional theoretical results on improving these bounds remains for future
work.

\section{Experiments}
\label{sec:experiments}

In the following subsections, we benchmark the proposed algorithm against a wide
variety of methods on two real-data applications. In
Section~\ref{sec:exp-traffic}, we revisit the traffic sensor network problem of
\cite{Srinivas:2010}. In Section~\ref{sec:exp-model}, we consider the problem of
automatic model selection and algorithm configuration.

\subsection{Application to a traffic sensor network}
\label{sec:exp-traffic}

In this experiment, we are given data taken from traffic speed sensors deployed
along highway I-880 South in California. Traffic speeds were collected at
$K=357$ sensor locations for all working days between 6AM and 11AM for an entire
month.  Our task is to identify the single location with the highest expected
speed, i.e.\ the least congested. This data was also used in the work of
\cite{Srinivas:2010}.

\begin{figure}
\centering
\includegraphics[width=0.48\textwidth]{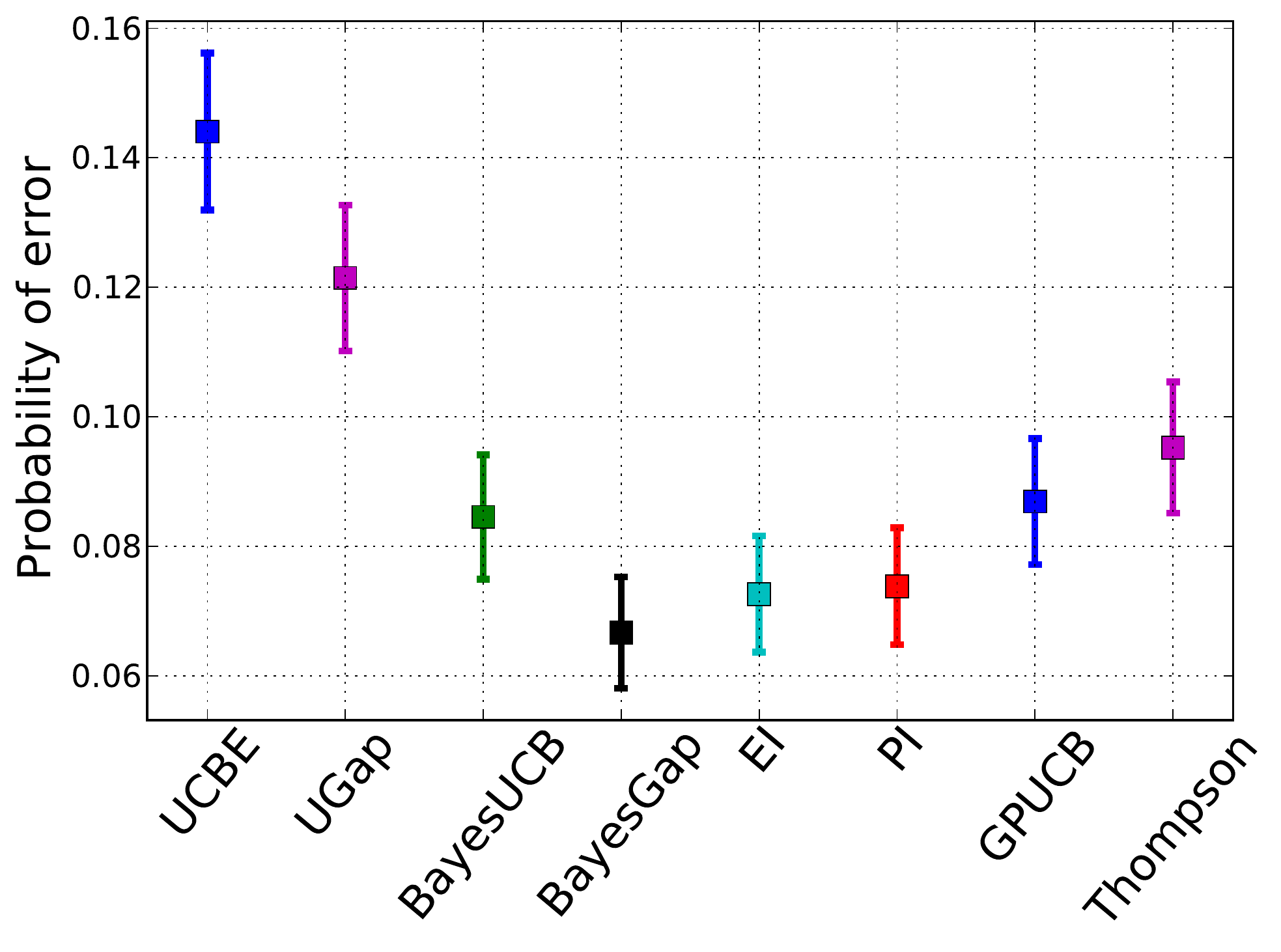}
\caption{\footnotesize Probability of error on the optimization domain of
traffic speed sensors. For this real data set, BayesGap provides considerable
improvements over the Bayesian cumulative regret alternatives and the
frequentist simple regret counterparts.}
\label{fig:traffic}
\end{figure}

Naturally, the readings from different sensors are correlated, however, this
correlation is not necessarily only due to geographical location. Therefore
specifying a similarity kernel over the space of traffic sensor locations alone
would be overly restrictive. Following the approach of \cite{Srinivas:2010}, we
construct the design matrix treating two-thirds of the available data as
historical and use the remaining third to evaluate the policies. In more detail,
The GP kernel matrix $G\in\R^{K\times K}$ is set to be empirical covariance
matrix of measurements for each of the $K$ sensor locations. As explained in
Section 4, the corresponding design matrix is $X = VD^{\frac12}$, where
$G=VDV^T$.

Following \cite{Srinivas:2010}, we estimate the noise level $\sigma$ of the
observation model using this data.  We consider the average empirical variance
of each individual sensor (i.e.\ the signal variance corresponding to the
diagonal of $G$) and set the noise variance $\sigma^2$ to 5\% of this value;
this corresponds to $\sigma^2=4.78$. We choose a broad prior with regularization
coefficient $\eta=20$.

In order to evaluate different bandit and Bayesian optimization algorithms, we
use each of the remaining 840 sensor signals (the aforementioned third of the
data) as the true mean vector $\mu$ for independent runs of the experiment. Note
that using the model in this way enables us to evaluate the ground truth for
each run (given by $\mu$, but not observed by the algorithm), and estimate the
actual probability that the policies return the best arm.

In this experiment, as well as in the next one, we estimate the hardness
parameter $H_\epsilon$ using the adaptive procedure outlined at the end of
Section 5.

We benchmark the proposed algorithm (BayesGap) against the following methods:

\noindent
\textbf{(1) {UCBE}:} Introduced by
    \cite{audibert-best}; this is a variant of the classical UCB policy of
    \cite{auer-ucb} that replaces the $\log(t)$ exploration term of UCB with a
    constant of order $\log(T)$ for known horizon $T$.

\noindent
\textbf{(2) {UGap}:} A gap-based exploration approach
    introduced by \cite{gabillon-unified}.

\noindent
\textbf{(3) {BayesUCB}} and \textbf{{GPUCB}:} Bayesian extensions of UCB which
    derive their confidence bounds from the posterior. Introduced by
    \cite{kaufmann-bayesucb} and \cite{Srinivas:2010} respectively.

\noindent
\textbf{(4) {Thompson sampling}:} A randomized, Bayesian index strategy
    wherein the $k$th arm is selected with probability given by a single-sample
    Monte Carlo approximation to the posterior probability that the arm is the
    maximizer \citep{chapelle-thompson,kaufmann-thompson,Agrawal:2013}.

\noindent
\textbf{(5) {Probability of Improvement (PI)}:} A classic Bayesian
    optimization method which selects points based on their probability of
    improving upon the current incumbent.

\noindent
\textbf{(6) {Expected Improvement (EI)}:} A Bayesian optimization, related to
    PI, which selects points based on the expected value of their
    improvement.

Note that techniques (1) and (2) above attack the problem of best arm
identification and use bounds which encourage more aggressive exploration.
However, they do not take correlation into account. On the other hand,
techniques such ad (3) are designed for cumulative regret, but model the
correlation among the arms. It might seem at first that we are comparing apples
and oranges. However, the purpose of comparing these methods, even if their
objectives are different, is to understand empirically what aspects of these
algorithms matter the most in practical applications.

The results, shown in Figure~\ref{fig:traffic}, are the probabilities of error
for each strategy, using a time horizon of $T=400$. (Here we used $\epsilon=0$,
but varying this quantity had little effect on the performance of each
algorithm.) By looking at the results, we quickly learn that techniques that
model correlation perform better than the techniques designed for best arm
identification, even when they are being evaluated in a best arm identification
task. The important conclusion is that one must always invest effort in
modelling the correlation among the arms.

The results also show that BayesGap does better than alternatives in this
domain. This is not surprising because BayesGap is the only competitor that
addresses budgets, best arm identification and correlation simultaneously.

\subsection{Automatic machine learning}
\label{sec:exp-model}

There exist many machine learning toolboxes, such as \texttt{Weka} and
\texttt{scikit-learn}. However, for a great many data practitioners interested
in finding the best technique for a predictive task, it is often hard to
understand what each technique in the toolbox does. Moreover, each technique can
have many free hyper-parameters that are not intuitive to most users.

Bayesian optimization techniques have already been proposed to automate machine
learning approaches, such as MCMC inference \citep{Mahendran:2012, Hamze:2013,
Wang:2013}, deep learning \citep{Bergstra:2011}, preference learning
\citep{Brochu:2007,Brochu:2010}, reinforcement learning and control
\citep{martinez-cantin:2007,Lizotte:2012}, and more \citep{snoek:2012b}. In
fact, methods to automate entire toolboxes (\texttt{Weka}) have appeared very
recently \citep{HutHooLey12-ParallelAC}, and go back to old proposals for
classifier selection \citep{Maron94Moore}.

Here, we will demonstrate BayesGap by automating regression with
\texttt{scikit-learn}. Our focus will be on minimizing the cost of
cross-validation in the domain of big data. In this setting, training and
testing each model can take a prohibitive long time. If we are working under a
finite budget, say if we only have three days before a conference deadline or
the deployment of a product, we cannot afford to try all models in all
cross-validation tests. However, it is possible to use techniques such as
BayesGap and Thompson sampling to find the best model with high probability. In
our setting, the action of ``pulling an arm'' will involve selecting a model,
splitting the dataset randomly into training and test sets, training the model,
and recording the test-set performance.

In this bandit domain, our arms will consist of five \texttt{scikit-learn}
techniques and associated parameters selected on a discrete grid. We consider
the following methods for regression:
\emph{Lasso (8 models)} with regularization parameters
    \texttt{alpha} = (0.0001, 0.0005, 0.001, 0.005, 0.01, 0.05, 0.1, 0.5),
\emph{Random Forests (64 models)} where we vary the number of trees,
    \text{\texttt{n\_estimators}}=(1,10,100,1000), the minimum number of
    training examples in a node to split
    \text{\texttt{min\_samples\_split}}=(1,3,5,7) and the minimum number of
    training examples in a leaf \text{\texttt{min\_samples\_leaf}}=(2,6,10,14),
\emph{linSVM (16 models)} where we vary the penalty parameter \text{\texttt{C}}=
    (0.001, 0.01, 0.1, 1) and the tolerance parameter
    \text{\texttt{epsilon}}=(0.0001, 0.001, 0.01, 0.1),
\emph{rbfSVM (64 models)} where we use the same grid as above for
    \texttt{C} and \texttt{epsilon}, and we add a third parameter which is the
    length scale $\gamma$ of the RBF kernel used by the SVM
    $\text{\texttt{gamma}}= (0.025, 0.05, 0.1, 0.2)$, and
\emph{K-nearest neighbors (8 models)} where we vary number of neighbors
    $\text{\texttt{n\_neighbors}}=(1,3,5,7,9,11,13,15).$ The total number of
    models is 160.
Within a class of regressors, we model correlation using a squared
exponential kernel with unit length scale, i.e., $k(x,x')=\text e^{-(x-x')^2}$. Using
this kernel, we compute a kernel matrix $G$ and construct the design matrix as before.

\begin{figure}[t!]
\centering
\includegraphics[width=0.36\textwidth]{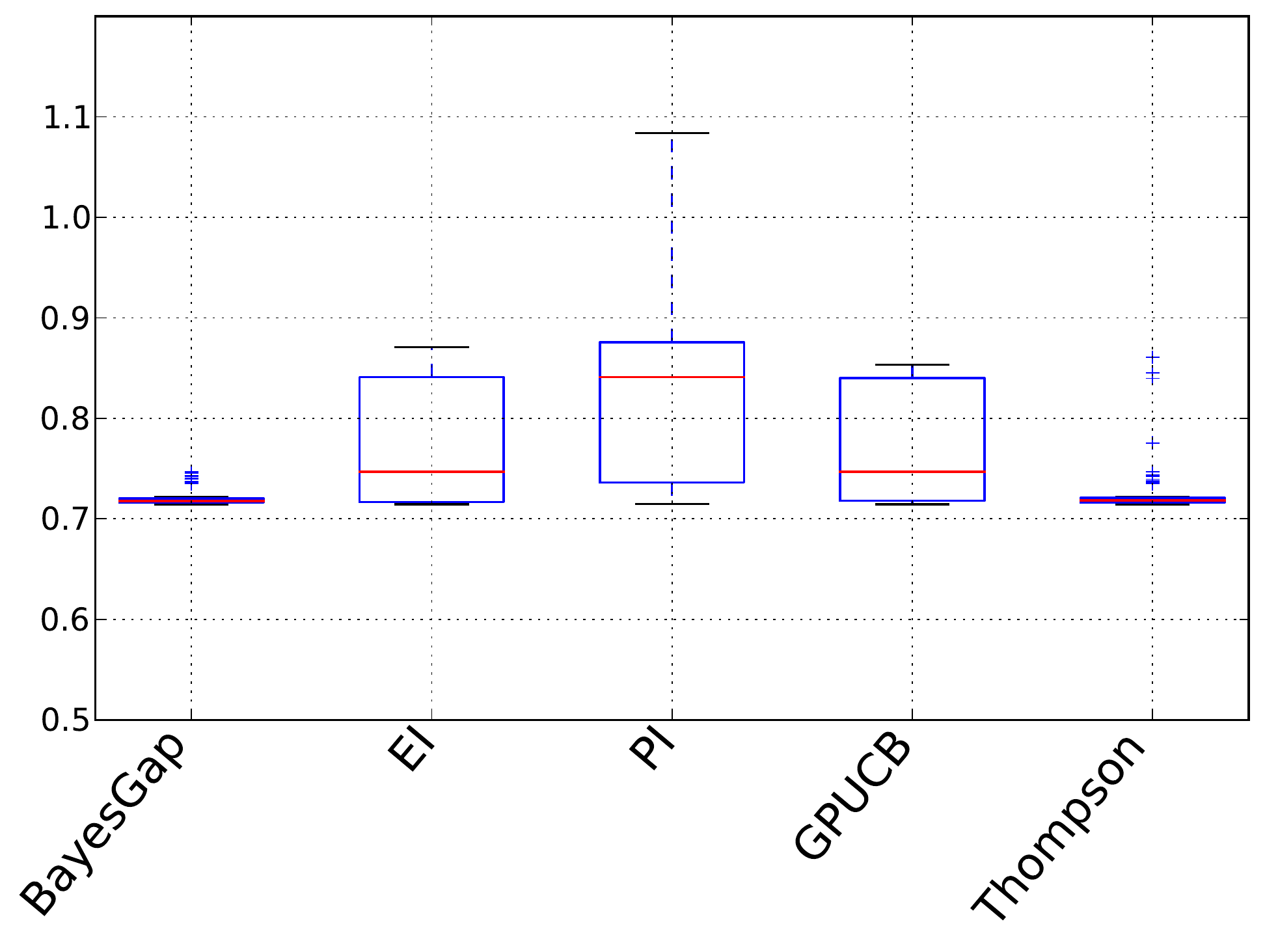}
\caption{\footnotesize Boxplot of RMSE over 100 runs with a fixed budget of
$T=10$.  EI, PI, and GPUCB get stuck in local minima. Note: lower is better.}
\label{fig:rmse-boxplot}
\end{figure}
\begin{figure}[t!]
\centering
\includegraphics[width=0.4\textwidth, trim=100 10 30 10, clip=true]{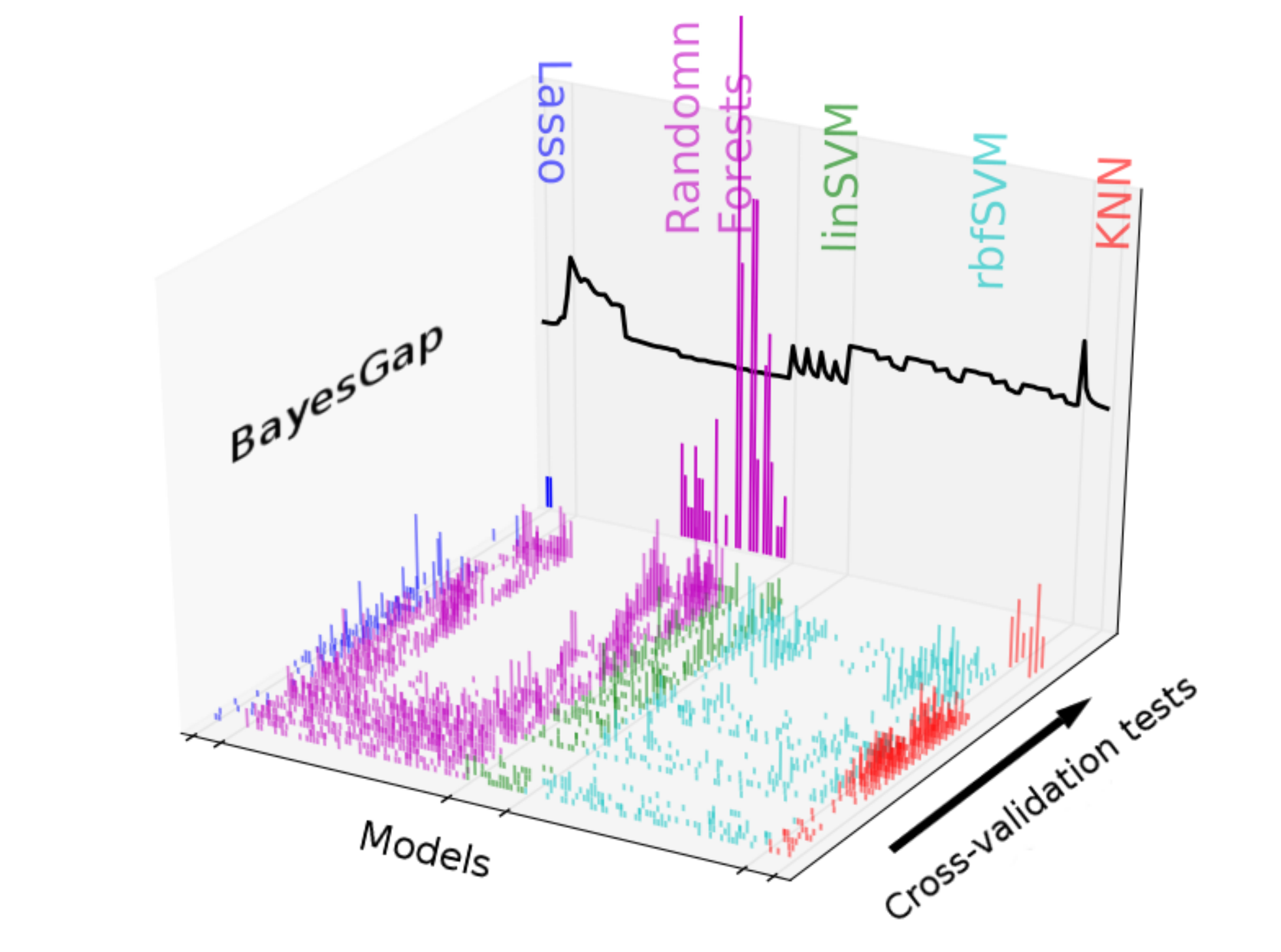}
\vspace{.5em}

\includegraphics[width=0.4\textwidth, trim=100 10 30 10, clip=true]{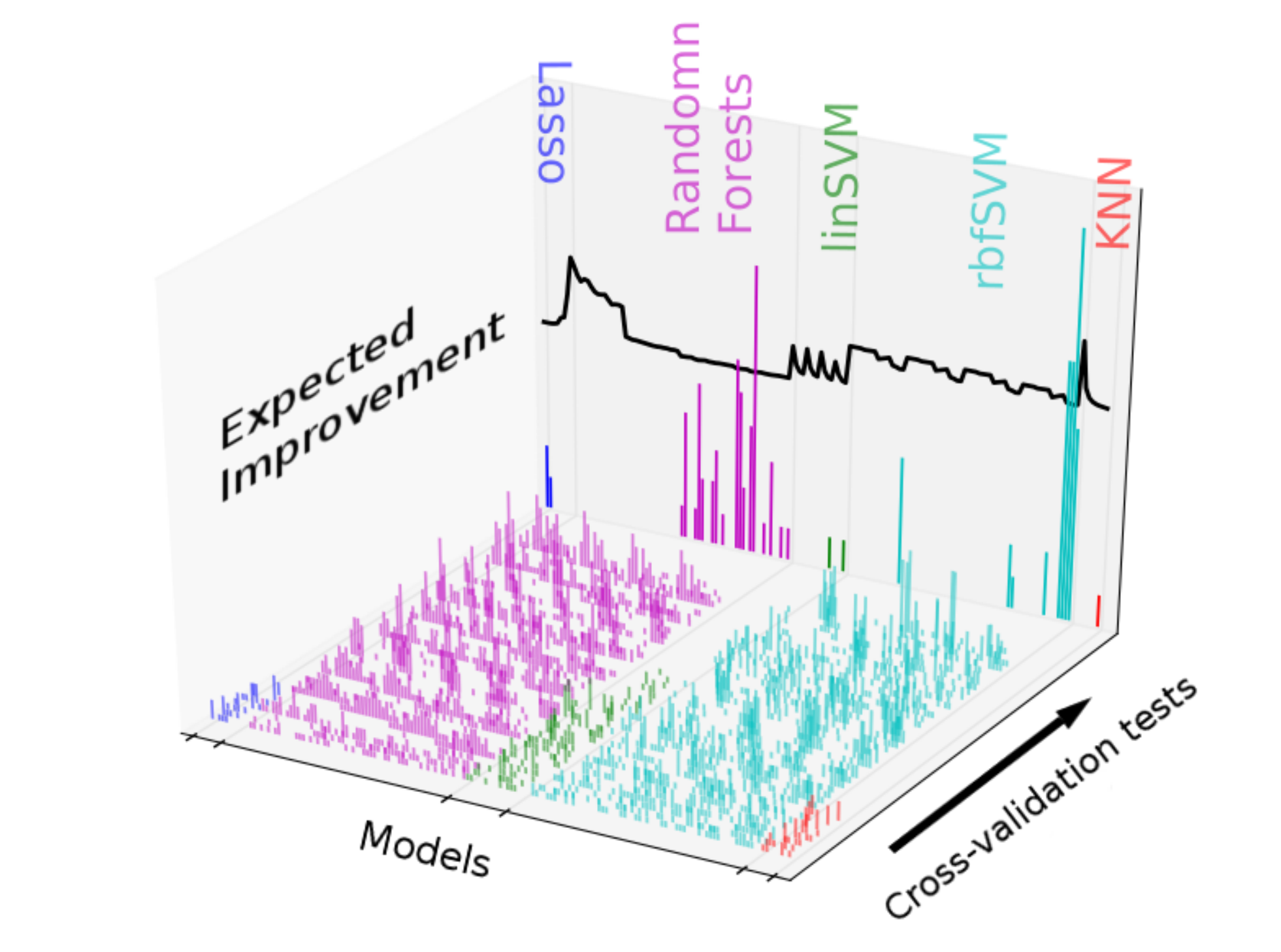}
\caption{\footnotesize Allocations and recommendations of BayesGap (top) and EI (bottom) over
100 runs at a budget of $T=40$ training and validation tests, and for 160 models
(i.e., more arms than possible observations). Histograms along the floor of the
plot show the arms pulled at each round while the histogram on the far wall
shows the final arm recommendation over 100 different runs. The solid black line
on the far wall shows the estimated ``ground truth'' RMSE for each model. Note
that EI quite often gets stuck in a locally optimal rbfSVM.}
\label{fig:arms-pulled}
\end{figure}

When an arm is pulled we select training and test sets that are each 10\% of the
size of the original, and ignore the remaining 80\% for this particular arm
pull. We then train the selected model on the training set, and test on the test
set. This specific form of cross-validation is similar to that of repeated
learning-testing \citep{arlot-cv, burman-rlt}.

We use the \texttt{wine} dataset from the UCI Machine Learning Repository, where
the task is to predict the quality score (between 0 and 10) of a wine given 11
attributes of its chemistry. We repeat the experiment 100 times.  We report, for
each method, an estimate of the RMSE for the recommended models on each run.
Unlike in the previous section, we do not have the ground truth generalization
error, and in this scenario it is difficult to estimate the actual ``probability
of error''.  Instead we report the RMSE, but remark that this is only a proxy
for the error rate that we are interested in.

The performance of the final recommendations for each strategy and a fixed
budget of $T=10$ tests is shown in Figure~\ref{fig:rmse-boxplot}. The results
for other budgets are almost identical. \emph{It must be emphasized that the
number of allowed function evaluations (10 tests) is much smaller than the
number of arms (160 models). Hence, frequentist approaches that require pulling
all arms, e.g. UGap, are inapplicable in this domain. }

The results indicate that Thompson and BayesGap are the best choices for this
 domain. Figure~\ref{fig:arms-pulled} shows the individual arms pulled and
 recommended by BayesGap (above) and EI (bottom), over each of the 100 runs, as
 well as an estimate of the ground truth RMSE for each individual model. EI and
 PI often get trapped in local minima. Due to the randomization inherent to
 Thompson sampling, it explores more, but in a more uniform manner (possibly
 explaining its poor results in the previous experiment).

\section{Conclusion}

We proposed a Bayesian optimization method for best arm identification with a
fixed budget. The method involves modelling of the correlation structure of the
arms via Gaussian process kernels. As a result of combining all these elements,
the proposed method outperformed techniques that do not model correlation or
that are designed for different objectives (typically cumulative regret). This
strategy opens up room for greater automation in practical domains with budget
constraints, such as the automatic machine learning application described in
this paper.

Although we focused on a Bayesian treatment of the UGap algorithm, the same
approach could conceivably be applied to other techniques such as UCBE. As
demonstrated by \cite{Srinivas:2010} and in this paper, it is possible to easily
show that the Bayesian bandits obtain similar bounds as the frequentist methods.
However, in our case, we conjecture that much stronger bounds should be possible
if we consider all the information brought in by the priors and measurement
models.

{\small
\bibliographystyle{abbrvnat}
\bibliography{bayesgap}

\begin{thebibliography}{40}
\providecommand{\natexlab}[1]{#1}
\providecommand{\url}[1]{\texttt{#1}}
\expandafter\ifx\csname urlstyle\endcsname\relax
  \providecommand{\doi}[1]{doi: #1}\else
  \providecommand{\doi}{doi: \begingroup \urlstyle{rm}\Url}\fi

\bibitem[Agrawal and Goyal(2013)]{Agrawal:2013}
S.~Agrawal and N.~Goyal.
\newblock Thompson sampling for contextual bandits with linear payoffs.
\newblock In \emph{ICML}, 2013.

\bibitem[Arlot and Celisse(2010)]{arlot-cv}
S.~Arlot and A.~Celisse.
\newblock A survey of cross-validation procedures for model selection.
\newblock \emph{Statistics Surveys}, 4:\penalty0 40--79, 2010.

\bibitem[Audibert et~al.(2010)Audibert, Bubeck, and Munos]{audibert-best}
J.-Y. Audibert, S.~Bubeck, and R.~Munos.
\newblock Best arm identification in multi-armed bandits.
\newblock In \emph{Conference on Learning Theory}, 2010.

\bibitem[Auer et~al.(2002)Auer, Cesa-Bianchi, and Fischer]{auer-ucb}
P.~Auer, N.~Cesa-Bianchi, and P.~Fischer.
\newblock Finite-time analysis of the multiarmed bandit problem.
\newblock \emph{Machine Learning}, 47\penalty0 (2):\penalty0 235--256, 2002.

\bibitem[Azimi et~al.(2011)Azimi, Fern, and Fern]{Azimi:2011}
J.~Azimi, A.~Fern, and X.~Fern.
\newblock Budgeted optimization with concurrent stochastic-duration
  experiments.
\newblock In \emph{NIPS}, pages 1098--1106, 2011.

\bibitem[Bergstra et~al.(2011)Bergstra, Bardenet, Bengio, and
  K{\'e}gl]{Bergstra:2011}
J.~Bergstra, R.~Bardenet, Y.~Bengio, and B.~K{\'e}gl.
\newblock Algorithms for hyper-parameter optimization.
\newblock In \emph{NIPS}, pages 2546--2554, 2011.

\bibitem[Brochu et~al.(2007)Brochu, {de Freitas}, and Ghosh]{Brochu:2007}
E.~Brochu, N.~{de Freitas}, and A.~Ghosh.
\newblock Active preference learning with discrete choice data.
\newblock In \emph{NIPS}, pages 409--416, 2007.

\bibitem[Brochu et~al.(2010{\natexlab{a}})Brochu, Brochu, and
  de~Freitas]{Brochu:2010}
E.~Brochu, T.~Brochu, and N.~de~Freitas.
\newblock A {Bayesian} interactive optimization approach to procedural
  animation design.
\newblock In \emph{ACM SIGGRAPH/Eurographics Symposium on Computer Animation},
  pages 103--112, 2010{\natexlab{a}}.

\bibitem[Brochu et~al.(2010{\natexlab{b}})Brochu, Cora, and {de
  Freitas}]{brochu-tutorial}
E.~Brochu, V.~Cora, and N.~{de Freitas}.
\newblock A tutorial on {B}ayesian optimization of expensive cost functions.
\newblock Technical Report arXiv:1012.2599, arXiv, 2010{\natexlab{b}}.

\bibitem[Bubeck et~al.(2009)Bubeck, Munos, and Stoltz]{bubeck-pure}
S.~Bubeck, R.~Munos, and G.~Stoltz.
\newblock Pure exploration in multi-armed bandits problems.
\newblock In \emph{International Conference on Algorithmic Learning Theory},
  2009.

\bibitem[Burman(1989)]{burman-rlt}
P.~Burman.
\newblock A comparative study of ordinary cross-validation, v-fold
  cross-validation and the repeated learning-testing methods.
\newblock \emph{Biometrika}, 76\penalty0 (3):\penalty0 pp. 503--514, 1989.

\bibitem[Cesa-Bianchi and Lugosi(2006)]{cesa-bianchi-book}
N.~Cesa-Bianchi and G.~Lugosi.
\newblock \emph{Prediction, Learning, and Games}.
\newblock Cambridge University Press, New York, 2006.

\bibitem[Chapelle and Li(2012)]{chapelle-thompson}
O.~Chapelle and L.~Li.
\newblock An empirical evaluation of {Thompson} sampling.
\newblock In \emph{NIPS}, 2012.

\bibitem[de~Freitas et~al.(2012)de~Freitas, Smola, and Zoghi]{zoghi-detbo}
N.~de~Freitas, A.~Smola, and M.~Zoghi.
\newblock {Exponential Regret Bounds for Gaussian Process Bandits with
  Deterministic Observations}.
\newblock In \emph{ICML}, 2012.

\bibitem[Gabillon et~al.(2011)Gabillon, Ghavamzadeh, Lazaric, and
  Bubeck]{gabillon-multi}
V.~Gabillon, M.~Ghavamzadeh, A.~Lazaric, and S.~Bubeck.
\newblock Multi-bandit best arm identification.
\newblock In \emph{NIPS}, 2011.

\bibitem[Gabillon et~al.(2012)Gabillon, Ghavamzadeh, and
  Lazaric]{gabillon-unified}
V.~Gabillon, M.~Ghavamzadeh, and A.~Lazaric.
\newblock Best arm identification: A unified approach to fixed budget and fixed
  confidence.
\newblock In \emph{NIPS}, 2012.

\bibitem[Hamze et~al.(2013)Hamze, Wang, and de~Freitas]{Hamze:2013}
F.~Hamze, Z.~Wang, and N.~de~Freitas.
\newblock Self-avoiding random dynamics on integer complex systems.
\newblock \emph{ACM Transactions on Modelling and Computer Simulation},
  23\penalty0 (1):\penalty0 9:1--9:25, 2013.

\bibitem[Hennig and Schuler(2012)]{hennig-entropy}
P.~Hennig and C.~Schuler.
\newblock Entropy search for information-efficient global optimization.
\newblock \emph{JMLR}, 13:\penalty0 1809--1837, 2012.

\bibitem[Hoffman et~al.(2011)Hoffman, Brochu, and de~Freitas]{Hoffman:2011}
M.~W. Hoffman, E.~Brochu, and N.~de~Freitas.
\newblock Portfolio allocation for {Bayesian} optimization.
\newblock In \emph{UAI}, pages 327--336, 2011.

\bibitem[Hutter et~al.(2011)Hutter, Hoos, and Leyton-Brown]{Hutter:smac}
F.~Hutter, H.~H. Hoos, and K.~Leyton-Brown.
\newblock Sequential model-based optimization for general algorithm
  configuration.
\newblock In \emph{Proceedings of LION-5}, page 507–523, 2011.

\bibitem[Jones(2001)]{Jones:2001}
D.~Jones.
\newblock A taxonomy of global optimization methods based on response surfaces.
\newblock \emph{J. of Global Optimization}, 21\penalty0 (4):\penalty0 345--383,
  2001.

\bibitem[Kaufmann et~al.(2012{\natexlab{a}})Kaufmann, Capp{\'e}, and
  Garivier]{kaufmann-bayesucb}
E.~Kaufmann, O.~Capp{\'e}, and A.~Garivier.
\newblock On {Bayesian} upper confidence bounds for bandit problems.
\newblock In \emph{AIStats}, 2012{\natexlab{a}}.

\bibitem[Kaufmann et~al.(2012{\natexlab{b}})Kaufmann, Korda, and
  Munos]{kaufmann-thompson}
E.~Kaufmann, N.~Korda, and R.~Munos.
\newblock {Thompson} sampling: an asymptotically optimal finite-time analysis.
\newblock In \emph{International Conference on Algorithmic Learning Theory},
  2012{\natexlab{b}}.

\bibitem[Kohavi et~al.(2009)Kohavi, Longbotham, Sommerfield, and
  Henne]{kohavi-abtesting}
R.~Kohavi, R.~Longbotham, D.~Sommerfield, and R.~Henne.
\newblock Controlled experiments on the web: survey and practical guide.
\newblock \emph{Data Mining and Knowledge Discovery}, 18:\penalty0 140--181,
  2009.

\bibitem[Lizotte et~al.(2012)Lizotte, Greiner, and Schuurmans]{Lizotte:2012}
D.~J. Lizotte, R.~Greiner, and D.~Schuurmans.
\newblock An experimental methodology for response surface optimization
  methods.
\newblock \emph{Journal of Global Optimization}, 53\penalty0 (4):\penalty0
  699--736, 2012.

\bibitem[Mahendran et~al.(2012)Mahendran, Wang, Hamze, and {de
  Freitas}]{Mahendran:2012}
N.~Mahendran, Z.~Wang, F.~Hamze, and N.~{de Freitas}.
\newblock Adaptive {MCMC} with {Bayesian} optimization.
\newblock \emph{Journal of Machine Learning Research - Proceedings Track},
  22:\penalty0 751--760, 2012.

\bibitem[Maron and Moore(1994)]{Maron94Moore}
O.~Maron and A.~W. Moore.
\newblock Hoeffding races: Accelerating model selection search for
  classification and function approximation.
\newblock In \emph{NIPS}, pages 59--66, 1994.

\bibitem[{Martinez-Cantin} et~al.(2007){Martinez-Cantin}, {de Freitas}, Doucet,
  and Castellanos]{martinez-cantin:2007}
R.~{Martinez-Cantin}, N.~{de Freitas}, A.~Doucet, and J.~A. Castellanos.
\newblock Active policy learning for robot planning and exploration under
  uncertainty.
\newblock 2007.

\bibitem[Mo{\v c}kus(1982)]{Mockus:1982}
J.~Mo{\v c}kus.
\newblock The {B}ayesian approach to global optimization.
\newblock In \emph{Systems Modeling and Optimization}, volume~38, pages
  473--481. Springer, 1982.

\bibitem[Munos(2011)]{Munos:2011}
R.~Munos.
\newblock Optimistic optimization of a deterministic function without the
  knowledge of its smoothness.
\newblock In \emph{NIPS}, pages 783--791, 2011.

\bibitem[Murphy(2012)]{Murphy:2012}
K.~P. Murphy.
\newblock \emph{Machine learning: A probabilistic perspective}.
\newblock Cambridge, MA, 2012.

\bibitem[Scott(2010)]{scott-bandits}
S.~Scott.
\newblock A modern {Bayesian} look at the multi-armed bandit.
\newblock \emph{Applied Stochastic Models in Business and Industry},
  26\penalty0 (6), 2010.

\bibitem[Snoek et~al.(2011)Snoek, Larochelle, and Adams]{snoek-oppcost}
J.~Snoek, H.~Larochelle, and R.~P. Adams.
\newblock Opportunity cost in {Bayesian} optimization.
\newblock In \emph{Neural Information Processing Systems Workshop on Bayesian
  Optimization}, 2011.

\bibitem[Snoek et~al.(2012)Snoek, Larochelle, and Adams]{snoek:2012b}
J.~Snoek, H.~Larochelle, and R.~Adams.
\newblock Practical {Bayesian} optimization of machine learning algorithms.
\newblock In \emph{NIPS}, pages 2960--2968, 2012.

\bibitem[Srinivas et~al.(2010)Srinivas, Krause, Kakade, and
  Seeger]{Srinivas:2010}
N.~Srinivas, A.~Krause, S.~M. Kakade, and M.~Seeger.
\newblock Gaussian process optimization in the bandit setting: No regret and
  experimental design.
\newblock In \emph{ICML}, pages 1015--1022, 2010.

\bibitem[Thornton et~al.(2013)Thornton, Hutter, Hoos, and
  Leyton-Brown]{HutHooLey12-ParallelAC}
C.~Thornton, F.~Hutter, H.~H. Hoos, and K.~Leyton-Brown.
\newblock Auto-{WEKA}: Combined selection and hyperparameter optimization of
  classification algorithms.
\newblock In \emph{KDD}, pages 847--855, 2013.

\bibitem[Valko et~al.(2013)Valko, Carpentier, and Munos]{Valko:SSOO}
M.~Valko, A.~Carpentier, and R.~Munos.
\newblock Stochastic simultaneous optimistic optimization.
\newblock In \emph{ICML}, 2013.

\bibitem[Villemonteix et~al.(2009)Villemonteix, Vazquez, and
  Walter]{villemonteix-iago}
J.~Villemonteix, E.~Vazquez, and E.~Walter.
\newblock An informational approach to the global optimization of
  expensive-to-evaluate functions.
\newblock \emph{Journal of Global Optimization}, 44\penalty0 (4):\penalty0
  509--534, 2009.

\bibitem[Wang et~al.(2013{\natexlab{a}})Wang, Mohamed, and
  de~Freitas]{Wang:2013}
Z.~Wang, S.~Mohamed, and N.~de~Freitas.
\newblock Adaptive {Hamiltonian} and {Riemann} manifold {Monte Carlo} samplers.
\newblock In \emph{ICML}, 2013{\natexlab{a}}.

\bibitem[Wang et~al.(2013{\natexlab{b}})Wang, Zoghi, Matheson, Hutter, and {de
  Freitas}]{Wang:rembo}
Z.~Wang, M.~Zoghi, D.~Matheson, F.~Hutter, and N.~{de Freitas}.
\newblock Bayesian optimization in high dimensions via random embeddings.
\newblock In \emph{IJCAI}, 2013{\natexlab{b}}.

\end{thebibliography}
}

\appendix
\newpage
\section{Theorem~\ref{theorem:budget-bound}}

The proof of this section and the lemmas of the next section follow from the
proofs of \cite{gabillon-unified}. The modifications we have made to this proof
correspond to the introduction of the function $g_k$ which bounds the
uncertainty $s_k$ in order to make it simpler to introduce other models. We also
introduce a sufficient condition on this bound, i.e.\ that it is monotonically
decreasing in $N$ in order to bound the arm pulls with respect to $g_k^{-1}$.
Ultimately, this form of the theorem reduces the problem of of proving a regret
bound to that of checking a few properties of the uncertainty model.

\begin{theorem}
    \label{theorem:budget-bound}
    Consider a bandit problem with horizon $T$ and $K$ arms. Let $U_k(t)$ and
    $L_k(t)$ be upper and lower bounds that hold for all times $t\leq T$ and all
    arms $k\leq K$ with probability $1-\delta$. Finally, let $g_k$ be a
    monotonically decreasing function such that $s_k(t)\leq g_k(N_k(t-1))$ and
    $\sum_k g_k^{-1}(H_{k\epsilon})\leq T-K$. We can then bound the simple
    regret as
       \begin{equation}
        \Pr(R_{\Omega_T} \leq \epsilon) \geq 1 - KT\delta.
    \end{equation}
\end{theorem}
\begin{proof}
We will first define the event $\E$ such that on this event every mean
is bounded by its associated bounds for all times $t$. More precisely
we can write this as
\begin{equation*}
    \E = \{
    \forall k\leq K, \forall t\leq T,
    L_k(t) \leq \mu_k \leq U_k(t)\}.
\end{equation*}
By definition, these bounds are given such that the probability of
deviating from a single bound is $\delta$. Using a union bound we can
then bound the probability of remaining within all bounds as
$\Pr(\E)\geq 1-KT\delta$.

We will next condition on the event $\E$ and assume regret of the form
$R_{\Omega_T}>\epsilon$ in order to reach a contradiction. Upon
reaching said contradiction we can then see that the simple regret must
be bounded by $\epsilon$ with probability given by the probability of
event $\E$, as stated above. As a result we need only show that a
contradiction occurs.

We will now define $\tau=\argmin_{t\leq T} B_{J(t)}(t)$ as the time at
which the recommended arm attains the minimum bound, i.e.\
$\Omega_T=J(\tau)$ as defined in \eqref{eqn:budget-omega}. Let $t_k\leq
T$ be the last time at which arm $k$ is pulled. Note that each arm must
be pulled at least once due to the initialization phase. We can then
show the following sequence of inequalities:
\begin{align*}
    \min(0, s_k(t_k)-\Delta_k) + s_k(t_k)
    &\geq B_{J(t_k)}(t_k) \tag{a} \\
    &\geq B_{\Omega_T}(\tau) \tag{b} \\
    &\geq R_{\Omega_T} \tag{c} \\
    &> \epsilon. \tag{d}
\end{align*}
Of these inequalities, (a) holds by Lemma~\ref{lemma:upper-bound-B}, (c)
holds by Lemma~\ref{lemma:arm-regret-bound}, and (d) holds by our
assumption on the simple regret. The inequality (b) holds due to the
definition $\Omega_T$ and time $\tau$. Note, that we can also write the
preceding inequality as two cases
\begin{align*}
    s_k(t_k) > 2s_k(t_k)-\Delta_k > \epsilon,
    &\quad\text{if } \Delta_k > s_k(t_k); \\
    2s_k(t_k)-\Delta_k \geq s_k(t_k) > \epsilon,
    &\quad\text{if } \Delta_k \leq s_k(t_k)
\end{align*}
This leads to the following bound on the confidence diameter,
\begin{equation*}
    s_k(t_k)
    > \max(\tfrac12(\Delta_k+\epsilon),\epsilon)
    = H_{k\epsilon}
\end{equation*}
which can be obtained by a simple manipulation of the above equations.
More precisely we can notice that in each case, $s_k(t_k)$
upper bounds both $\epsilon$ and $\tfrac12(\Delta_k+\epsilon)$, and thus
it obviously bounds their maximum.

Now, for any arm $k$ we can consider the final number of arm pulls,
which we can write as
\begin{align*}
    N_k(T)
    = N_k(t_k-1) + 1
    &\leq g^{-1}(s_k(t_k)) + 1 \\
    &< g^{-1}(H_{k\epsilon}) + 1.
\end{align*}
This holds due to the definition of $g$ as a monotonic decreasing
function, and the fact that we pull each arm at least once during the
initialization stage. Finally, by summing both sides with respect to $k$
we can see that $\sum_k g^{-1}(H_{k\epsilon}) + K > T$, which
contradicts our definition of $g$ in the Theorem statement.
\end{proof}

\section{Lemmas}

\setcounter{lemma}{0}
\renewcommand{\thelemma}{\Alph{section}\arabic{lemma}}
\renewcommand{\thecorollary}{\Alph{section}\arabic{lemma}}

In order to simplify notation in this section, we will first introduce
$B(t)=\min_k B_k(t)$ as the minimizer over all gap indices for any time
$t$. We will also note that this term can be rewritten as
\begin{equation*}
    B(t) = B_{J(t)}(t) = U_{j(t)}(t) - L_{J(t)}(t),
\end{equation*}
which holds due to the definitions of $j(t)$ and $J(t)$.

\begin{lemma}
    \label{lemma:arm-regret-bound}
    For any sub-optimal arm $k\neq k^*$, any time $t\in\{1,\dots,T\}$,
    and on event $\E$, the immediate regret of pulling that arm is upper
    bounded by the index quantity, i.e. $B_k(t)\geq R_k$.
\end{lemma}
\begin{proof}
    We can start from the definition of the bound and expand this term
    as
    \begin{align*}
        B_k(t)
        &= \max_{i\neq k} U_i(t) - L_k(t) \\
        &
        \geq \max_{i\neq k} \mu_i - \mu_k
        = \mu^* - \mu_k = R_k.
    \end{align*}
    The first inequality holds due to the assumption of event $\E$,
    whereas the following equality holds since we are only considering
    sub-optimal arms, for which the best alternative arm is obviously
    the optimal arm.
\end{proof}

\begin{lemma}
    \label{lemma:bound-order}
    For any time $t$ let $k=a_t$ be the arm pulled, for which the
    following statements hold:
    \begin{align*}
        &
        \textrm{if } k=j(t), \textrm{ then }
        L_{j(t)}(t) \leq L_{J(t)}(t), \\
        &
        \textrm{if } k=J(t), \textrm{ then }
        U_{j(t)}(t) \leq U_{J(t)}(t).
    \end{align*}
\end{lemma}
\begin{proof}
    We can divide this proof into two cases based on which of the two
    arms is selected.

    \textbf{Case 1:} let $k=j(t)$ be the arm selected. We will then
    assume that $L_{j(t)}(t)>L_{J(t)}(t)$ and show that this is a
    contradiction. By definition of the arm selection rule we know that
    $s_{j(t)}(t)\geq s_{J(t)}(t)$, from which we can easily deduce that
    $U_{j(t)}(t)>U_{J(t)}(t)$ by way of our first assumption. As a
    result we can see that
    \begin{align*}
        B_{j(t)}(t)
        &= \max_{j\neq j(t)} U_j(t) - L_{j(t)}(t) \\
        &< \max_{j\neq J(t)} U_j(t) - L_{J(t)}(t)
         = B_{J(t)}(t).
    \end{align*}
    This inequality holds due to the fact that arm $j(t)$ must
    necessarily have the highest upper bound over all arms. However,
    this contradicts the definition of $J(t)$ and as a result it must
    hold that $L_{j(t)}(t)\leq L_{J(t)}(t)$.

    \textbf{Case 2:} let $k=J(t)$ be the arm selected. The proof follows
    the same format as that used for $k=j(t)$.
\end{proof}

\begin{corollary}
    \label{cor:uncertainty}
    If arm $k=a_t$ is pulled at time $t$, then the minimum index is
    bounded above by the uncertainty of arm $k$, or more precisely
    \begin{equation*}
        B(t) \leq s_k(t).
    \end{equation*}
\end{corollary}
\begin{proof}
    We know that $k$ must be restricted to the set $\{j(t),J(t)\}$ by
    definition. We can then consider the case that $k=j(t)$, and by
    Lemma~\ref{lemma:bound-order} we know that this imposes an order on
    the lower bounds of each possible arm, allowing us to write
    \begin{align*}
        B(t)
        \leq
        U_{j(t)}(t) - L_{j(t)}(t) = s_{j(t)}(t)
    \end{align*}
    from which our corollary holds. We can then easily see that a
    similar argument holds for $k=J(t)$ by ordering the upper bounds,
    again via Lemma~\ref{lemma:bound-order}.
\end{proof}

\begin{lemma}
    \label{lemma:upper-bound-B}
    On event $\E$, for any time $t\in\{1,\dots,T\}$, and for arm $k=a_t$
    the following bound holds on the minimal gap,
    \begin{equation*}
        B(t) \leq \min(0, s_k(t) - \Delta_k) + s_k(t).
    \end{equation*}
\end{lemma}
\begin{proof}
    In order to prove this lemma we will consider a number of cases
    based on which of $k\in\{j(t), J(t)\}$ is selected and whether or
    not one or neither of these arms corresponds to the optimal arm
    $k^*$. Ultimately, this results in six cases, the first three of
    which we will present are based on selecting arm $k=j(t)$.

    \textbf{Case 1:} consider $k^*=k=j(t)$. We can then see that the
    following sequence of inequalities holds,
    \begin{equation*}
        \mu_{(2)}
        \stackrel{(a)}\geq \mu_{J(t)}(t)
        \stackrel{(b)}\geq L_{J(t)}(t)
        \stackrel{(c)}\geq L_{j(t)}(t)
        \stackrel{(d)}\geq \mu_k - s_k(t).
    \end{equation*}
    Here (b) and (d) follow directly from event $\E$ and (c) follows
    from Lemma~\ref{lemma:bound-order}. Inequality (a) follows trivially
    from our assumption that $k=k^*$, as a result $J(t)$ can only be as
    good as the 2nd-best arm. Using the definition of $\Delta_k$ and the
    fact that $k=k^*$, the above inequality yields
    \begin{equation*}
        s_k(t) - (\mu_k - \mu_{(2)})
        = s_k(t) - \Delta_k \geq 0
    \end{equation*}
    Therefore the $\min$ in the result of
    Lemma~\ref{lemma:upper-bound-B} vanishes
    and the result follows from Corollary~\ref{cor:uncertainty}.

    \textbf{Case 2:} consider $k=j(t)$ and $k^*=J(t)$. We can then write
    \begin{align*}
        B(t)
        &= U_{j(t)}(t) - L_{J(t)}(t) \\
        &\leq \mu_{j(t)}(t) + s_{j(t)}(t) - \mu_{J(t)}(t) + s_{J(t)}(t) \\
        &\leq \mu_k - \mu^* + 2s_k(t)
        \intertext{where the first inequality holds from event $\E$,
        and the second holds because by definition the selected arm
        must have higher uncertainty. We can then simplify this as}
        &= 2s_k(t) - \Delta_k \\
        &\leq \min(0, s_k(t) - \Delta_k) + s_k(t),
    \end{align*}
    where the last step evokes Corollary~\ref{cor:uncertainty}.

    \textbf{Case 3:} consider $k=j(t)\neq k^*$ and $J(t)\neq k^*$. We
    can then write the following sequence of inequalities,
    \begin{equation*}
        \mu_{j(t)}(t) + s_{j(t)}(t)
        \stackrel{(a)}\geq U_{j(t)}(t)
        \stackrel{(b)}\geq U_{k^*}(t)
        \stackrel{(c)}\geq \mu^*.
    \end{equation*}
    Here (a) and (c) hold due to event $\E$ and (b) holds since by
    definition $j(t)$ has the highest upper bound other than $J(t)$,
    which in turn is not the optimal arm by assumption in this case. By
    simplifying this expression we obtain $s_k(t)-\Delta_k\geq 0$, and
    hence the result follows from Corollary~\ref{cor:uncertainty} as in
    Case 1.

    \textbf{Cases 4--6:} consider $k=J(t)$. The proofs for these three
    cases follow the same general form as the above cases and is
    omitted. Cases 1 through 6 cover all possible scenarios and prove
    Lemma~\ref{lemma:upper-bound-B}.
\end{proof}

\begin{lemma}
    \label{lemma:gaussian-deviation}
    Consider a normally distributed random variable $X\sim\mathcal
    N(\mu,\sigma^2)$ and $\beta\geq 0$. The probability that $X$ is
    within a radius of $\beta\sigma$ from its mean can then be
    written as
    \begin{equation*}
        \Pr\big(|X - \mu| \leq \beta\sigma\big)
        \geq 1-e^{-\beta^2/2}.
    \end{equation*}
\end{lemma}
\begin{proof}
    Consider $Z\sim\Norm(0,1)$. The probability that $Z$ exceeds some positive
    bound $c>0$ can be written
    \begin{align*}
        \Pr(Z>c)
        &=
        \frac{e^{-c^2/2}}{\sqrt{2\pi}}
        \int_c^{\infty} e^{(c^2-z^2)/2}\,dz \\
        &=
        \frac{e^{-c^2/2}}{\sqrt{2\pi}}
        \int_c^{\infty} e^{-(z-c)^2/2 - c(z-c)}\,dz \\
        &\leq
        \frac{e^{-c^2/2}}{\sqrt{2\pi}}
        \int_c^{\infty} e^{-(z-c)^2/2}\,dz
        = \tfrac12 e^{-c^2/2}.
    \end{align*}
    The inequality holds due to the fact that $e^{-c(z-c)}\leq1$ for
    $z\geq c$. Using a union bound we can then bound both sides as
    $\Pr(|Z|>c)\leq e^{-c^2/2}$. Finally, by setting $Z=(X-\mu)/\sigma$
    and $c=\beta$ we obtain the bound stated above.
\end{proof}

\end{document}